\documentclass[10pt,twocolumn,letterpaper]{article}

\usepackage{wacv}
\usepackage{times}
\usepackage{epsfig}
\usepackage{graphicx}
\usepackage{amsmath}
\usepackage{amssymb}
\usepackage{booktabs}
\usepackage{caption}
\usepackage{amsthm}
\usepackage{diagbox}
\usepackage{slashbox}
\usepackage{multirow}
\usepackage{balance}
\def\wacvPaperID{1440} % Enter the WACV Paper ID here

\wacvalgorithmstrack   % Uncomment this line if you are submitting to the Algorithms Track.
\wacvfinalcopy % *** Uncomment this line for the final submission

\ifwacvfinal
\usepackage[breaklinks=true,bookmarks=false]{hyperref}
\else
\usepackage[pagebackref=true,breaklinks=true,colorlinks,bookmarks=false]{hyperref}
\fi

\begin{document}

%%%%%%%%% TITLE
\title{Accumulated Trivial Attention Matters in Vision Transformers on Small Datasets}

\author{Xiangyu Chen\textsuperscript{\rm 1}, Qinghao Hu\textsuperscript{\rm 2}, Kaidong Li\textsuperscript{\rm 1}, Cuncong Zhong\textsuperscript{\rm 1}, Guanghui Wang\textsuperscript{\rm 3}$^*$\\
\textsuperscript{\rm 1}Department of EECS, University of Kansas, KS, USA \\ 
\textsuperscript{\rm 2}Institute of Automation, Chinese Academy of Sciences, China \\ 
\textsuperscript{\rm 3}Department of CS, Toronto Metropolitan University, Toronto, ON, Canada \\
\tt\small xychen@ku.edu, wangcs@ryerson.ca (* corresponding author)
}

\maketitle
\thispagestyle{empty}

%%%%%%%%% ABSTRACT
\begin{abstract}

Vision Transformers has demonstrated competitive performance on computer vision tasks benefiting from their ability to capture long-range dependencies with multi-head self-attention modules and multi-layer perceptron. However, calculating global attention brings another disadvantage compared with convolutional neural networks, \ie requiring much more data and computations to converge, which makes it difficult to generalize well on small datasets, which is common in practical applications. Previous works are either focusing on transferring knowledge from large datasets or adjusting the structure for small datasets. After carefully examining the self-attention modules, we discover that the number of trivial attention weights is far greater than the important ones and the accumulated trivial weights are dominating the attention in Vision Transformers due to their large quantity, which is not handled by the attention itself. This will cover useful non-trivial attention and harm the performance when trivial attention includes more noise, \eg in shallow layers for some backbones. To solve this issue, we proposed to divide attention weights into trivial and non-trivial ones by thresholds, then Suppressing Accumulated Trivial Attention (SATA) weights by proposed Trivial WeIghts Suppression Transformation (TWIST) to reduce attention noise. Extensive experiments on CIFAR-100 and Tiny-ImageNet datasets show that our suppressing method boosts the accuracy of Vision Transformers by up to $2.3\%$. Code is available at https://github.com/xiangyu8/SATA.

 \end{abstract}

%%%%%%%%% BODY TEXT
\section{Introduction}

%%%% intro of vision transformers 
Convolutional Neural Networks (CNN) have dominated computer vision tasks for the past decade, especially with the emergence of ResNet \cite{he2016deep}. Convolution operation, the core technology in CNN, takes all the pixels from its receptive field as input and outputs one value. When the layers go deep, the stacked locality becomes non-local as the receptive field of each layer is built on the convolution results of the previous layer. The advantage of convolution is its power to extract local features, making it converge fast and a good fit, especially for data-efficient tasks. Different from CNN, Vision Transformer (ViT) \cite{dosovitskiy2020image} and its variants \cite{chu2021twins,dong2022cswin,gajurel2021fine,liu2021swin,ma2021miti, sajid2021audio} consider the similarities between each image patch embedding and all other patch embeddings. This global attention boosts its potential for feature extraction, however, requiring a large amount of data to feed the model and limiting its application to small datasets. 

%%%% vit in small datasets
On the one hand, CNNs have demonstrated superior performance to ViT regarding the accuracy, computation and convergence speed on data-efficient tasks, like ResNet-50 for image classification \cite{cen2021deep,chen4179882increasing, liu2021efficient, lee2021vision}, object detection \cite{zhang2022dynamic} and ResNet-12 for few-shot learning \cite{chen2021meta}. However, to improve the performance is to find more inductive bias to include, which is tedious. The local attention also sets a lower performance ceiling by eliminating much necessary non-local attention, which is in contrast to Vision Transformers. On the other hand, the stronger feature extraction ability of Vision Transformers can perfectly make up for the lack of data on small datasets. As a result, Vision Transformers show promising direction for those tasks.

\begin{figure*}
    \centering
   \includegraphics[width=0.95\linewidth,height=0.5\linewidth]{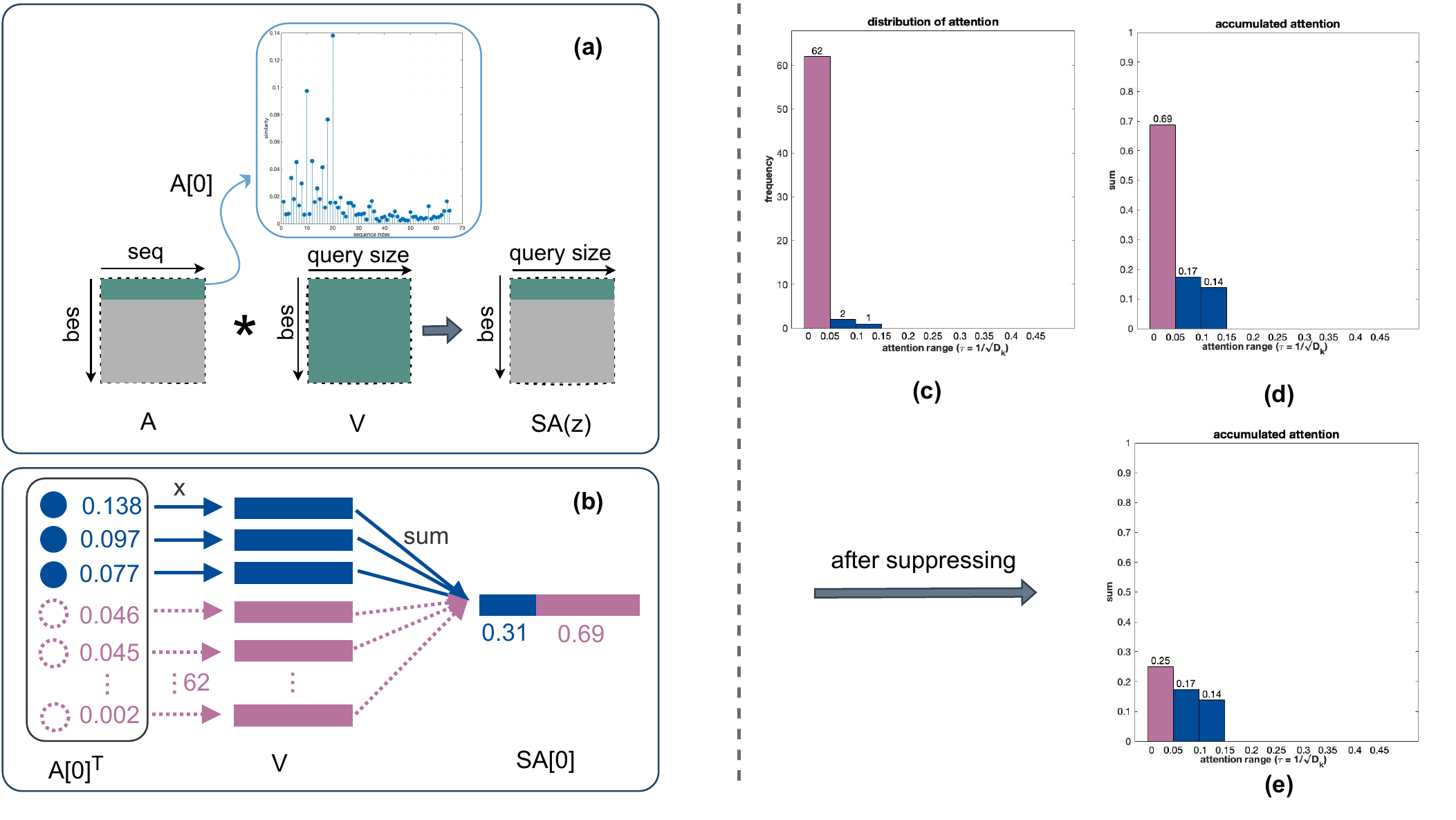}    
   \captionof{figure}{Our proposed SATA strategy. (a) The multi-head self-attention module in Vision Transformers. Each row in $A$ represents attention weights corresponding to all sequences in $\bf{v}$. (b) A closer look at how to get the first sequence $SA[0]$ after applying attention. We set the threshold to 0.05. The blue part denotes larger attention weights and purple is for trivial ones. We get up to 62 trivial attention weights and sum up to 0.69 in the entire attention in SA[0] compared with 0.31 from similar sequences.  (c) The distribution of attention weights. (d) Accumulated attention within each bin. (e) The result of suppressing trivial weights by our approach. Even if single attention trivial weight contributes little, the accumulated trivial attention is still dominating, which is harmful when the attention contains much noise as in shallow layers of some backbones.}
   \label{fig:head}
\end{figure*}%

To adapt Vision Transformers to data-efficient tasks, some researchers focus on transfer learning \cite{soroush2022sima,touvron2021training, wu2022tinyvit}, semi-supervised learning and unsupervised learning to leverage large datasets. Others dedicate to self-supervised learning or other modalities to dig the inherent structure information of images themselves \cite{chen2021few}. For supervised learning, one path is to integrate convolution operations in Vision Transformers to increase their locality. Another approach is to increase efficiency by revising the structure of Vision Transformers themselves \cite{patel2022aggregating}. The proposed method belongs to the second category.

%%%% our solution
The main transformer blocks include a multi-head self-attention (MHSA) module and a multi-layer perceptron (MLP) layer, along with some layer normalization, where the MHSA module is key to enriching each sequence by including long-range dependencies with all other sequences, \ie attention. Intuitively, this attention module is expected to have larger coefficients for those sequences with higher similarity while smaller values for those less similar as the example $A[0]$ in Figure \ref{fig:head}(a). In this way, all sequences can be enhanced by other similar sequences. However, this only considers single similarities themselves, but not their accumulation. Taking a closer look at the dot product operation on how each weighted sequence obtained in Figure \ref{fig:head}(a), we can find it is from weighting each sequence with attention coefficients and then summing up into one sequence as shown in Figure \ref{fig:head}(b). This is problematic when the sequence length is large and those less similar sequences are noise. When the similarities are added from all less similar sequences, the accumulated sum can be even greater than the largest similarity as in Figure \ref{fig:head}(d) caused by the small-value but large-amount trivial attention coefficients. This means the accumulated trivial attention dominates the attention, which brings much noise to the convergence of the Transformer. As a result, the trivial attention would hinder the training of the Transformer on small datasets.  To solve this rooted problem in Vision Transformers and make it better deploy on small datasets, we proposed to suppress all trivial attention and hence the accumulated trivial attention to make sequences with higher similarity dominant again. 

%%%% our contribution
The contributions of this paper are summarized below.
\begin{itemize}
\item We found the accumulated trivial attention inherently dominates the MHSA module in vision Transformers and brings many noises on shallow layers. To cure this problem, we propose Suppress Accumulated Trivial
Attention (SATA) to separate out trivial attention first and then decrease the selected attention. 
\item We propose a \textit{trivial weights suppression transformation (TWIST)} to control accumulated trivial attention. The proposed transformation is proved to suppress the trivial weights to a portion of maximum attention.
\item  Extensive experiments on CIFAR-100 and Tiny-ImageNet demonstrated up to 2.3$\%$ gain in accuracy by using the proposed method.
\end{itemize}

%-------------------------------------------------------------------------
\section{Related Work}

Vision Transformer has become a powerful counterpart of CNN in computer vision tasks since its introduction in 2020 \cite{dosovitskiy2020image}, benefiting from its power to capture long-term dependencies. This ability is brought by their inherent structures in ViT, including the MHSA attention module which enhances each sequence with all other sequences, and MLP layers to model the relationships across all sequences. Including global attention also has weaknesses, like requiring large datasets to train, unlike the local attention in CNN. However, such large datasets are not easily accessible in many cases considering both time and effort cost in labeling and maintaining, \eg rare diseases in the medical field. One direct solution is to search for more data, either borrowing data or knowledge from available large datasets and applying it to small datasets like transfer learning \cite{soroush2022sima} and distillation \cite{touvron2021training, wu2022tinyvit} or digging other information like self-supervised learning \cite{li2021efficient,caron2021emerging, he2022masked} and other modalities to exploit available labels \cite{radford2021learning,sajid2021audio}. Another folder is to adjust the structure of transformers. For instance, integrate convolutional layers to transformers to mitigate its rely on the amount of data like CvT \cite{wu2021cvt}, LeViT \cite{graham2021levit}, CMT \cite{guo2021cmt} and CeiT \cite{yuan2021incorporating}, design efficient attention modules to replace the quadratic computation complexity MHSA as Reformer \cite{kitaev2020reformer}, Swin \cite{liu2021swin}, Swin-v2 \cite{liu2022swin}, Twins \cite{chu2021twins}, HaloNet \cite{vaswani2021scaling} and Cswin \cite{dong2022cswin}, or remove MLP layers \cite{ding2022davit}.

% Attention related work
Regarding the MHSA module in vision Transformers, previous works can be divided into two paths according to the components in the attention function, input and the function itself. For the input of MHSA module, \ie $\bf{qk}^T$, Swin \cite{liu2021swin} and Swin-v2 \cite{liu2022swin} calculate attention within windows instead of full sequences, CvT \cite{wu2021cvt} uses convolutional layers to replace the linear layers to get $\bf{kqv}$. And there are also some works argue that the Softmax function in original vision transformers can be revised (\eg adding learnable temperature \cite{lee2021vision}) or even replaced with other functions (\eg \textit{l}$_1$ norm in SimA \cite{soroush2022sima} and Gaussian kernel in SOFT \cite{lu2021soft}). In this work, we post-process the results after the Softmax function, which can be categorized into the attention function folder. 

The attention module is designed to focus on more alike sequences and less on different ones. This is based on the premise that all sequences are clear and include little noise. In this way, attention can enhance each sequence with alike sequences and useful signals get emphasized. However, this does not hold true when the features contain much noise. As shown in Figure \ref{fig:head}, the example attention for one sequence $A[0]$ looks reasonable, with less similar attention weights but many trivial weights. However, when we check the sum of all trivial weights, it is far greater than the maximum attention. Just imagine all these sequences assigned trivial attention are harmful sequences, the accumulated attention from trivial weights is dominating the whole attention and even covers it. This happens when attention contains too much noise, like some shallow layers as mentioned in \cite{wei2021shallow}. For shallow layer features, the image is a natural signal and has low information density. In other words, the image itself contains much noise, like the background of an object. This noise can even extend to several shallow layers due to the limitation of current feature extraction models, making shallow layers contain much more noise than deeper ones. However, determining the boundary for ``shallow" is difficult since it is dependent on the depth of the model, feature extraction of the model, and the noisy degree of datasets. Thus, we designed a learnable suppressing scale $s$ for all layers, avoiding finding this boundary by adding little computation.

%-------------------------------------------------------------------------
\section{Methodology}
This section first introduces the MHSA module in Vision Transformers and its limitations, followed by the proposed suppressing steps, setting a threshold to separate out trivial attention coefficients and then decreasing their sum as shown in Figure \ref{fig:steps}.

\begin{figure}[t]
\begin{center}
   \includegraphics[width=0.9\linewidth]{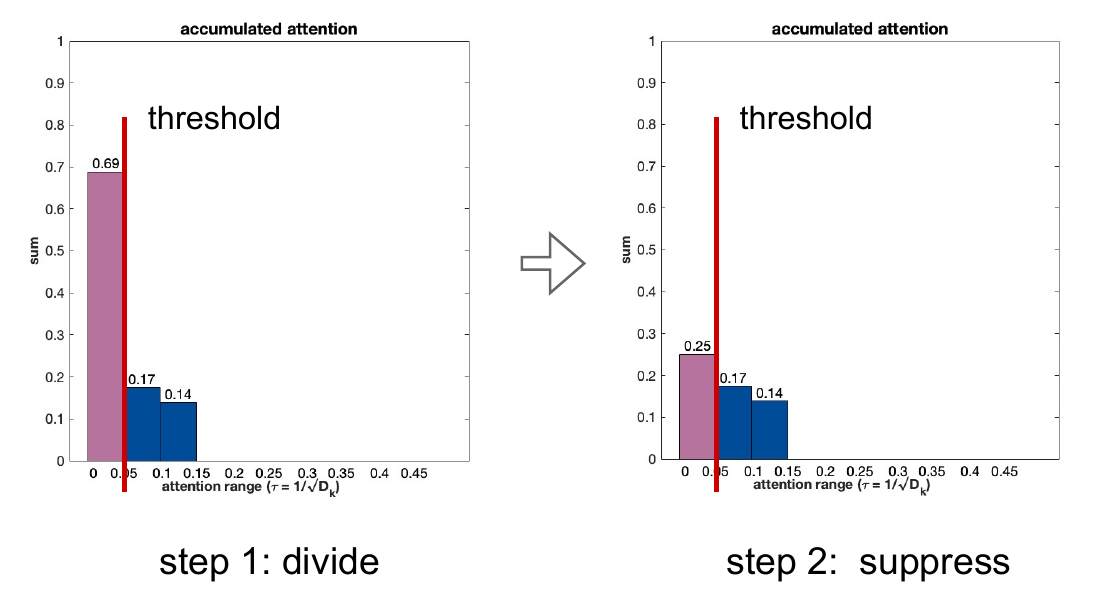}
\end{center}
   \caption{Our proposed suppressing method SATA. Step 1: divide trivial and non-trivial weights (0.05 in this example). Step 2: apply TWIST transformation on trivial weights.}
\label{fig:steps}
\end{figure}

\subsection{Revisit MHSA}
MHSA modules \cite{dosovitskiy2020image} in vision Transformers is the key to enriching sequence embedding in capturing long-range dependencies. To get this, input $\bf{z}\in\mathbb{R}^{N\times D}$, where $N$ is the length of sequence and $D$ is the dimension of sequence, is first passed through linear layers to get $\bf{q}$, $\bf{k}$ and $\bf{v}$. Then calculate the attention $A$ to weight each sequence.
\begin{equation}
[\bf{q},\bf{k},\bf{v}] = \bf{z}U_{qkv}
\end{equation}
\begin{equation}
A = softmax(\bf{qk}^T/\sqrt{D_h})
\end{equation}
Finally, compute a weighted sum across all sequences to get the final enhanced sequences, $SA(z)$ in Figure \ref{fig:head}(a).
\begin{equation}
SA(\bf{z}) = A\bf{v}
\label{eq:attn}
\end{equation}

Taking a closer look at the detailed operations in Equation (\ref{eq:attn}) as in Figure \ref{fig:head}(b), to get the first row $SA[0]$, $V$ is first weighted by all elements of the first row of $A$, and the weighted sequences are summed up to one sequence. One example of the attention weights $A[0]$ on the CIFAR-100 dataset can be found in Figure \ref{fig:head}(a), where several weighs are larger, indicating the corresponding sequence is more similar to the current query sequence and hence with higher attention, while most of them are small indicating less similarity and importance. 

We further explore the statistics of $A[0]$, the distribution of attention as in Figure \ref{fig:head} (c), where each bar denotes the number of attention weights in each similarity range. As in the example, up to 62 attention weights are below 0.05 indicating less similar sequences, while only 3 weights are greater than it. Surprisingly, after calculating the sum of attention weights in each bar and getting the graph in Figure \ref{fig:head} (d), the sum of trivial weights (below 0.05) is far larger than the sum of important weights, \ie SA[0] is composed of more information from trivial sequences than non-trivial sequences. In other words, trivial attention is dominating the MHSA altogether. This also means the MHSA modules bring more noise than information, making it converge slowly, especially on small datasets. To handle this, we need to first set a threshold to separate trivial/non-trivial attention weights and then suppress trivial ones. 

\subsection{Divide}
The first step is to get all trivial attention weights before suppressing them. Here we use a suppression threshold to divide the attention weights into trivial and non-trivial ones. Those attention weights below the suppression threshold are regarded as trivial weights, otherwise as non-trivial weights.
However, there are also some choices to set the threshold.

\textbf{Relative or absolute?} A relative threshold is a portion $t$ of a value, \eg the maximum attention weight $x_m$ in a row. \ie
\begin{equation}
threshold = t*x_m
\label{eq:thres1}
\end{equation}
Compared to this, an absolute threshold is a given value $t$. \ie 
\begin{equation}
threshold = t
\label{eq:thres2}
\end{equation}

Notice that the absolute attention weights depend on the length of the sequence, the function to get it (\eg temperature in Softmax function), datasets and so on. We set it a relative threshold as in Equation (\ref{eq:thres1}), where $t$ is the relative scale. This dividing process can be achieved by multiplying the original attention with a mask $M$ with `1' for trivial positions and `0' for nontrivial ones as below:
\begin{equation}
M = \left\{\begin{matrix}
1, & attention \le t*x_m \\ 
0,&  elsewhere
\end{matrix}\right.  
\end{equation}

Then the final attention after normalization can be obtained by multiplying the original attention with the mask $M$ as below.

\begin{equation}
A' = M \odot A
\label{eq:supp}
\end{equation}
where $\odot$ denotes element-wise product.

\subsection{Suppress}
To suppress the sum of trivial attention, we transform each trivial attention weight from $x_j$ to $x_j'$ and call this transformation as TWIST.
\newtheorem{lemma}{Lemma}
\begin{lemma}\label{lemma}
Given $n$ positive attention weights $x_1+x_2... +x_k+ x_{k+1}+...+x_{n-1}+x_m = 1$, where $x_1,...,x_k$ are trivial weights, less than a threshold $T$, $x_m$ is the maximum weight and $x_{k+1},...,x_{n-1}$ are the rest of weights. If
\begin{equation}
x_{j}{'}=\frac{x_{j}^{2}}{\sum_{i=1}^{k}x_{i}},  j = 1...k
\label{eq:supp}
\end{equation}
Then, 
\begin{equation}
{\sum_{i=1}^{k}x_{i}{'}} \leq x_m 
\label{eq:leq}
\end{equation}
and 
\begin{equation}
x_j' \leq x_j, j = 1..k 
\label{eq:comp}
\end{equation}
\end{lemma}
\begin{proof} Since for all $x_j$, $0<x_j\leq x_m\leq 1$ where $j = 1...k$, we have $x_j^{2}\leq x_j*x_m$. Then 
\begin{equation}
    {\sum_{j=1}^{k}x_{j}^{2}} \leq (x_1+x_2+...+x_j)*x_m
\end{equation}
Dividing both sides by the sum yields 
\begin{equation}
    {\sum_{j=1}^{k}x_{j}^{2}}/(x_1+x_2+...+x_j) \leq x_m
\end{equation}
or 
\begin{equation}
\frac{{\sum_{j=1}^{k}x_{j}^{2}}}{\sum_{i=1}^{k}x_{i}}\leq x_m
\end{equation}
Rewrite the numerator, 
\begin{equation}
\frac{x_{1}^{2}+x_{2}^{2}+...+x_{k}^{2}}{\sum_{i=1}^{k}x_{i}}\leq x_m
\label{eq:scale}
\end{equation}
which is exactly Equation (\ref{eq:leq}) after substituting Equation (\ref{eq:supp}). 

To prove Equation (\ref{eq:comp}), we only need to prove:
\begin{equation}
\frac{x_{j}^{2}}{\sum_{i=1}^{k}x_{i}}\leq x_j,  j = 1...k
\label{eq:comp1}
\end{equation}
Dividing both sides by positive value $x_j$ yields
\begin{equation}
\frac{x_{j}}{\sum_{i=1}^{k}x_{i}}\leq 1,  j = 1...k
\label{eq:comp2}
\end{equation}
As $x_j$ is one of the items in the denominator, Equation (\ref{eq:comp2}) holds true.
\end{proof}

Lemma 1 means if we want to make the sum of trivial weights less than the maximum weight, we can simply transform $x_j$ to $x_j{'}$ based on Equation (\ref{eq:supp}), and attention weights after transformation are always no more than original weights. We can also add a scale $s$ on both side of Equation (\ref{eq:scale}) to make the sum smaller than a portion $s$ ($s\ge$0) of the maximum. As a result, our final transformation on $x_1,...x_k$ to suppress the accumulated trivial attention weights is
\begin{equation}
x_{j}{'}=s*\frac{x_{j}^{2}}{\sum_{i=1}^{k}x_{i}},  j = 1...k
\label{eq:final}
\end{equation}
where the suppressing scale $s$ is learnable. We name this transformation in Equation \ref{eq:final} as \textit{trivial weights transformation (TWIST)}. This transformation can guarantee

\begin{equation}
    \sum_{i=1}^{k}x_{i}' \leq s*x_m
\end{equation}
%-------------------------------------------------------------------------
\section{Experiments}
This section presents experiment settings, results and discussions after implementing the suppressing of Vision Transformers.

\subsection{Settings}
We perform image classification on small datasets, including CIFAR-100 \cite{krizhevsky2009learning} and Tiny-ImageNet \cite{le2015tiny}. CIFAR-100 includes 60,000 images with size $32\times32$, 50,000 for train split and 10,000 for validation split. Tiny-ImageNet has 100,000 and 10,000 $64\times64$ images for train and validation split respectively. Following settings in \cite{lee2021vision}, we perform data augmentations including CutMix \cite{yun2019cutmix}, Mixup \cite{zhang2018mixup}, Auto Augment \cite{cubuk2019autoaugment}, Repeated Augment \cite{cubuk2020randaugment}, regularization including random erasing \cite{zhong2020random}, label smoothing \cite{szegedy2016rethinking} and stochastic depth \cite{huang2016deep}. Optimizer is also AdamW \cite{loshchilov2017decoupled}. The batch size is 128 and all models are trained for 100 epochs on one A100 GPU. The learning rate of model is set to 0.003 for ViT \cite{dosovitskiy2020image} and 0.001 for PiT \cite{heo2021rethinking}. The suppressing scale $s$ is initialized to 0.5 and its learning rate can be found in Table \ref{table:lr} \textit{lr2}. The threshold coefficient $t$ is fixed to 0.05 for PiT on CIFAR-100 and 0.1 for all other experiments.

\begin{table}
  \begin{center}
    {\small{
\begin{tabular}{cccc}
\toprule
& lr1 (model) & lr2 (CIFAR-100) & lr2 (T-ImageNet) \\ \midrule
ViT & 0.003 & $7e^{-5}$  &  0.001     \\    
PiT & 0.001 &  0.001 & $3e^{-4}$ \\
\bottomrule
\end{tabular}
}}
\end{center}
\caption{Learning rates for model (lr1) and suppressing scale $s$ (lr2) on CIFAR-100 and Tiny-ImageNet.}
\label{table:lr}
\end{table}

\subsection{Integrating with Vision Transformers}
To evaluate the effect of our proposed suppressing method, we integrate it with both the original ViT \cite{dosovitskiy2020image} and PiT \cite{heo2021rethinking} following the scale for small datasets in \cite{lee2021vision}, where the patch size is set to 4 for CIFAR-100 and 8 for Tiny-ImageNet, resulting in the same number of tokens 64 and 1 class token. From the results in Table \ref{table:all} we see that the accuracy for ViT is increased by $1.23\%$ on CIFAR-100 and up to $2.32\%$ on Tiny-ImageNet by integrating our trivial attention suppressing module. It also boosts PiT on both datasets with up to $1.21\%$ on CIFAR-100. These improvements demonstrate that taking care of trivial attention weights explicitly is necessary and suppressing them can improve performance. Besides this, examining the effect of our method on a different scale of tokens may be interesting future work.

\begin{table}
  \begin{center}
    {\small{
\begin{tabular}{lc|ll}
\toprule
Model & Param (M) & CIFAR-100 & T-ImageNet\\
\midrule
ResNet56$^*$ & 0.9 &76.36 & 58.77\\
ResNet110$^*$ & 1.7 & 79.86 & 62.96\\
EfficientNet B0$^*$ & 3.7 & 76.04 & 66.79\\ \hline
ViT &2.8 & 73.70 & 56.45\\
SATA-ViT (ours)&2.8 & \textbf{74.93}(\textit{\scriptsize {+1.23}}) & \textbf{58.77}(\textit{\scriptsize {+2.32}})  \\ \hline
PiT &  7.1 & 72.31 & 57.87 \\
SATA-PiT (ours) &7.1 & \textbf{73.52}(\textit{\scriptsize {+1.21}}) & \textbf{58.15}(\textit{\scriptsize {+0.28}}) \\ %\hline
\bottomrule
\end{tabular}
}}
\end{center}
\caption{Classification results on CIFAR-100 and Tiny-ImageNet dataset. Top 1 accuracy ($\%$) is reported.}
\label{table:all}
\end{table}

\subsection{Ablation study} % 2 components: learnable threshold and learnable suppressing degree
To verify how each module works, we decompose each component in the proposed suppressing module using ViT on Tiny-ImageNet. Specifically, to understand the necessity of suppressing, we let $s$ be a hyperparameter as $t$ and perform a grid search on both hyperparameters. The best accuracy and its search result are listed in Table \ref{table:ablation}. Comparing row 2 with row 0 in Table \ref{table:ablation}, we observe that suppressing brings $1.24\%$ more accuracy to ViT after grid search. In addition, we also set suppressing scale $s$ to 0 and find that accuracies for most threshold $t$ are near $56.45\%$ when no suppressing exists as shown in Table \ref{table:grid}. This indicates that those trivial attention weights are still helpful.

\begin{table}
  \begin{center}
    {\small{
\begin{tabular}{c|cc|l}
\toprule
index & suppress  & threshold & Top 1  \\ \hline
0   & -         & -         & 56.45  \\ \hline
1   & 0         & 0.075      & 57.47 \\ \hline
2   & 0.75      & 0.1       & 57.69 \\ \hline
3   & learnable & 0.1       & \textbf{58.77} \\
\bottomrule
\end{tabular}
}}
\end{center}
\caption{Ablation study. In this table, we compare suppressing with grid search $s$, suppressing to $0$, suppressing with learnable $s$ and no suppressing.}
\label{table:ablation}
\end{table}

\textbf{Grid Search.} For grid search, we select $s$ from $[1, 0.75, 0.5, 0.25, 0.1, 0]$ and $t$ from $[0.1, 0.05, 0.025, 0.01, 0]$. The learning rate is $0.003$, the same as ViT on CIFAR-100 and Tiny-ImageNet when no suppressing exists. In Table \ref{table:grid}, we can find the best result is from $s = 0.75$ and $t = 0.1$ on Tiny-ImageNet, while it is $s = 0$ and $t = 0.075$ for ViT on CIFAR-100 according to \ref{table:grid-cifar}, which means we get a better performance when removing those attention directly. We also select the initialization values for learnable $s$ from their average performance. From both Table \ref{table:grid} and Table \ref{table:grid-cifar}, we can see that the accuracies are increased in most cases with suppressing compared with the baseline when no suppressing happens, \ie when the last row $t = 0$ in both tables. The last column in Table \ref{table:grid} shows that deleting the attention will hurt the performance most of the time on Tiny-ImageNet while it helps all the time on CIFAR-100. We also observe that the relationship between $s$ and $t$ is complicated, neither linear nor inverse. This is reasonable since both parameters are highly correlated.

\begin{table}
  \begin{center}
    {\small{
\begin{tabular}{l|ccccc|c}
\toprule
\backslashbox{$t$}{$s$}  & 1     & 0.75  & 0.5   & 0.25  & 0.1   & 0     \\ \hline
0.1   & 56.43 & \textbf{57.69} & 56.22 & 57.30 & 56.46 & 56.51 \\ 
0.075   & 56.79 & 57.06 & 56.74 & 56.29 & 56.99 & 57.47 \\
0.05  & 56.11 & 56.72 & 56.47 & 56.87 & 57.24 & 56.47 \\ 
0.025 & 56.06 & 56.85 & 56.71 & 56.65 & 56.95 & 56.31 \\ 
0.01  & 57.23 & 56.31 & 56.63 & 56.45 & 56.44 & 56.55 \\ \hline
0 & \multicolumn{6}{c}{56.45} \\
\bottomrule
\end{tabular}
}}
\end{center}
\caption{Grid search on $s$ and $t$ for ViT on Tiny-ImageNet dataset. The baseline without suppressing is the last row when $t = 0$ and the last column denotes removing the attention directly.}
\label{table:grid}
\end{table}

\begin{table}
  \begin{center}
    {\small{
\begin{tabular}{l|ccccc|c}
\toprule
\backslashbox{$t$}{$s$}  & 1     & 0.75  & 0.5   & 0.25  & 0.1   & 0     \\ \hline
0.1   & 74.61 & 74.51 & 74.46 & 74.07 & 74.42 & 74.50 \\ 
0.075   & 74.81 & 74.01 & 73.89 & 74.65 & 73.97 & \textbf{74.75}\\
0.05  & 74.71 & 74.18 & 74.82 & 74.46 & 74.32 & 74.96 \\ 
0.025 & 74.34 & 73.39 & 73.93 & 74.34 & 74.21 & 73.86 \\ 
0.01  & 74.15 & 74.75 & 74.60 & 74.15 & 74.49 & 73.71 \\ \hline
0 & \multicolumn{6}{c}{73.70} \\
\bottomrule
\end{tabular}
}}
\end{center}
\caption{Grid search on $s$ and $t$ for ViT on CIFAR-100 dataset. The baseline without suppressing is the last row when $t = 0$ and the last column denotes removing the attention directly.}
\label{table:grid-cifar}
\end{table}

\textbf{Learnable $s$.} Learned $s$ of ViT on both CIFAR-100 and Tiny-ImageNet can be found in Figure \ref{fig:learnS}. Note that $s$ denotes the suppressing scale to the maximum attention for each sequence. In Table \ref{fig:learnS} we can see, for the first 2 layers on Tiny-ImageNet and the first 4 layers on CIFAR-100 dataset, the sum of trivial weights is ensured to be below the maximum. While for deep layers, the sum of trivial attention can be scaled up to several times the maximum. This is reasonable since for deeper layers, features contain less noise, and hence suppressing trivial attention weights is no longer necessary. Instead, another function of our SATA works is to adjust the distribution of attention by increasing trivial attention weights to be more comparable with the maximum and hence influence the distribution.

\begin{figure}[t]
\begin{center}
   \includegraphics[width=0.8\linewidth]{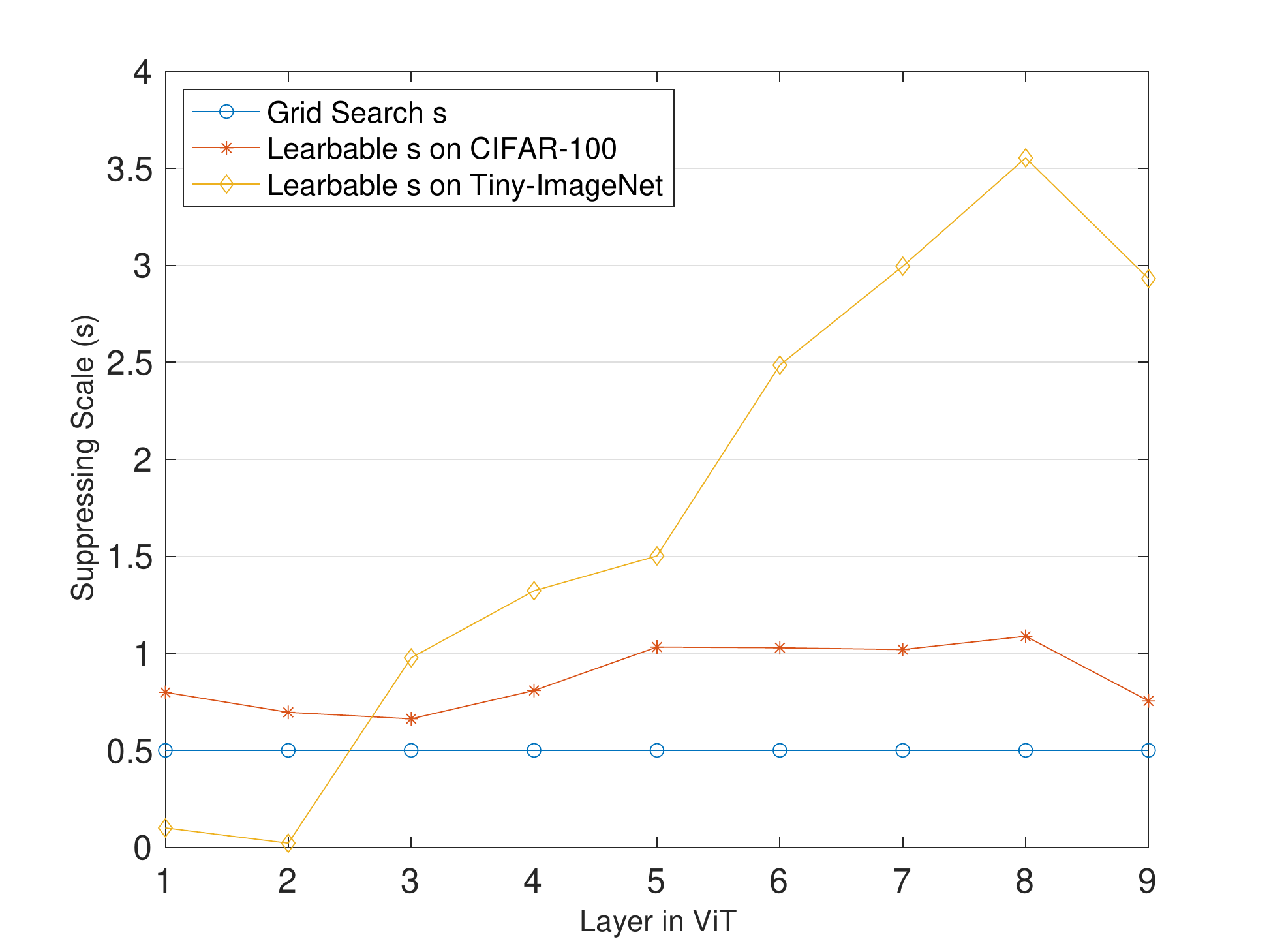}
\end{center}
   \caption{Learnable $s$ vs. fixed $s$ from grid search for ViT.}
\label{fig:learnS}
\end{figure}

\subsection{Comparison with different normalization}

\textbf{{Softmax with temperature.}} The original Softmax can be denoted by

\begin{equation}
A = softmax(\tau{\bf{qk}^T})
\end{equation}
where $\tau$ is the temperature to control the scale of the Softmax function. In the original ViT \cite{dosovitskiy2020image}, the normalization of attention modules is
\begin{equation}
A = softmax(\frac{\bf{qk}^T}{\sqrt{D_h}})
\end{equation}
where they use $1/\sqrt{D_h}$, the dimension of the head, as the temperature $\tau$ of the Softmax function to adjust the distribution of attention weights. The higher the temperature $\tau$, the sharper the Softmax function as shown in Figure \ref{fig:softmax}. Figure \ref{fig:softmax} shows that small values become even smaller and large ones are even larger when increasing the temperature of Softmax function from $\tau = 1$ to $\tau = 2$. This can also enlarge the ratio of larger values to smaller values, which is similar to the effect of our proposed suppressing trivial weights. However, increasing the temperature of Softmax cannot solve this. More specifically, it also brings side effects along with the suppression of trivial attention weights.

\begin{figure}[t]
\begin{center}
   \includegraphics[width=0.8\linewidth]{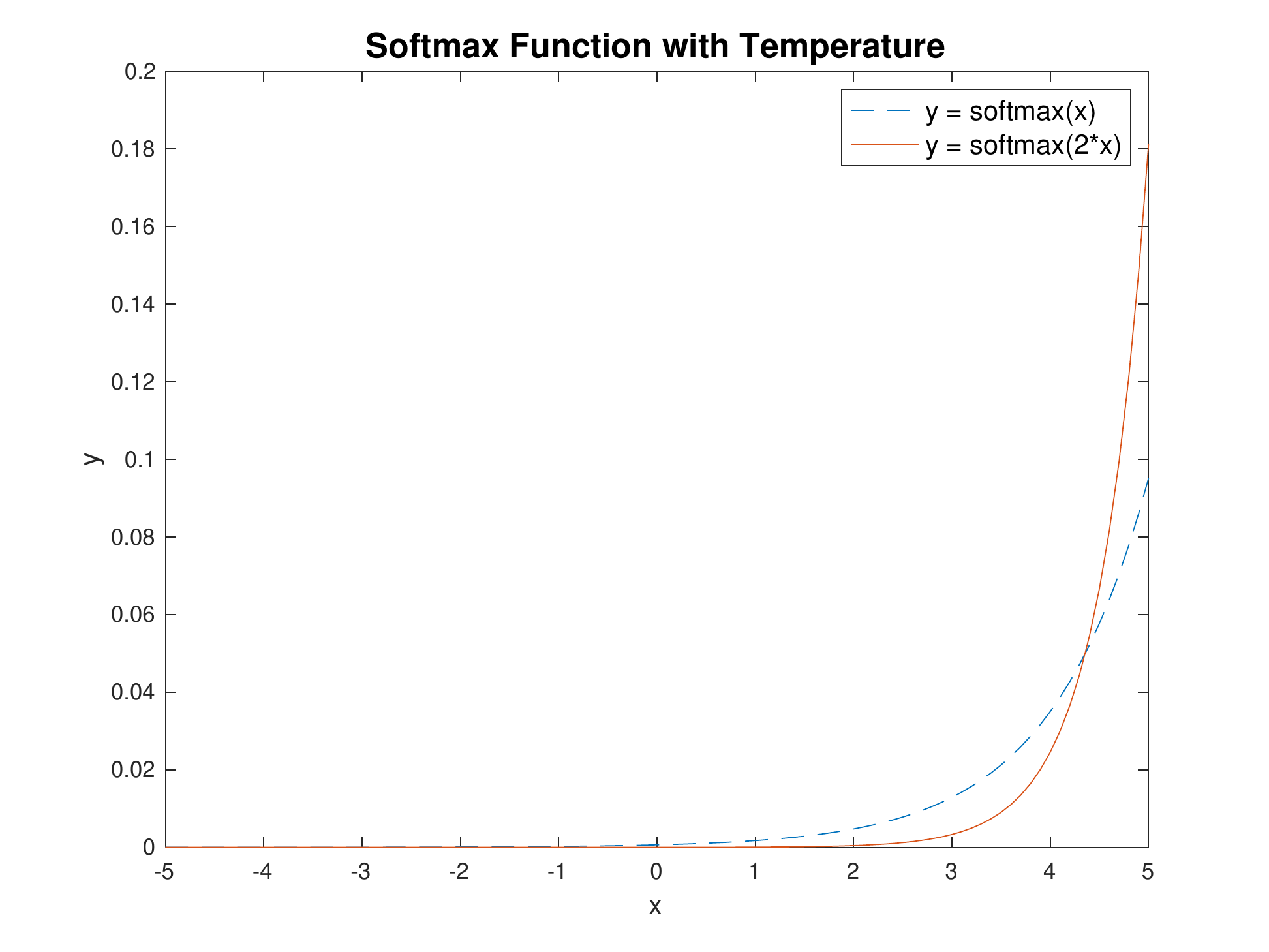}
\end{center}
   \caption{Softmax function with temperature. With the increase of temperature, the difference between large values and smaller values gets enlarged by decreasing smaller values and increasing greater ones.}
\label{fig:softmax}
\end{figure}

To check this, we first conduct experiments by setting different temperatures. The results can be found in Table \ref{table:tau}.
Table \ref{table:tau} shows that, compared with the Softmax with default temperature $\tau = 1/\sqrt{D_h}$ in original ViT, increasing the temperature to $2\times$, $3\times$ and $4\times$ all improves the performance on Tiny-ImageNet dataset. And the best result is from $\tau = 4/\sqrt{D_h}$ Tiny-ImageNet.

In summary, Softmax with higher temperature can mitigate the dominating accumulated trivial attention effect partly at the cost of changing the distribution of sensitive larger attention weights. While our proposed suppressing method decouples trivial and non-trivial attention weights to solve the dominating effect, making us available to take advantage of both adjusting trivial attention weights according to noise level and adjusting the distribution of non-trivial attention weights with higher temperature freely. The results to build our SATA on softmax with temperature are shown in Table \ref{table:tau}. According to Table \ref{table:tau}, the performance for Softmax with temperature is improved with the increase of $\tau$, while it decreases after applying our proposed module. In most cases \eg $\tau = 1\times, 2\times, 3\times$ of default temperature, SATA module further boosts the accuracy. This indicates that the performance increase of Softmax with temperature is from the enlargement of the gap between larger and smaller values, which is good for shallow layers while harming deeper layers. Our module can adjust trivial attention weights in both situations, noisy or noiseless, and especially when both happen in the same model while requiring different handling.

\begin{table}
  \begin{center}
    {\small{
\begin{tabular}{l|cccc}
\toprule
model & $\tau' = 1$ & $\tau' = 2$ &$\tau' = 3$& $\tau' = 4$ \\ \hline
 ViT + $\tau$  & 56.45  & 57.20 & 57.21 & $\textbf{57.81}$ \\ \hline
 ViT + $\tau$ + SATA  & $\textbf{58.77}$ & 58.12 & 57.72 & 57.58  \\ 
\bottomrule
\end{tabular}
}}
\end{center}
\caption{Combining Softmax with different $\tau = \tau' \times \frac{1}{\sqrt{D_h}}$ on Tiny-ImageNet. We adjust learning rate for $s$ after adding $\tau$. The best results are reported.}
\label{table:tau}
\end{table}

\textbf{Diagonal suppressing.} This module LSA is proposed in \cite{lee2021vision} considering that the attention in diagonal is from self-attention in the MHSA module for Vision Transformers, which is not necessary since the skip connection in the MHSA module will add the attention itself with a larger ratio compared with the self-attention in the attention branch. However, this self-attention usually is larger than other attention, leaving less room for other attention to get large values. To this end, they propose to manually set the diagonal to an extremely small value, making the attention after Softmax small. As the goals of this module and our SATA module are different, we can combine both methods together to yield better attention. The results are shown in Table \ref{table:LSA}. According to this table, the performance is increased on both CIFAR-100 and Tiny-ImageNet after implementing the LSA module and our module further adds up to $0.46\%$ on Tiny-ImageNet.

\begin{table}
  \begin{center}
    {\small{
\begin{tabular}{l|cc}
\toprule
model & CIFAR-100& Tiny-ImageNet  \\ \hline
ViT  & 73.70 & 56.45  \\ \hline
ViT + LSA  & 75.40      & 57.82\\ \hline
ViT + LSA + SATA  & 75.47      & 58.28 \\
\bottomrule
\end{tabular}
}}
\end{center}
\caption{Integrating with LSA module.}
\label{table:LSA}
\end{table}

%------------------------------------------------------------------------
\section{Conclusion}

In this paper, we have examined the MHSA modules in Vision Transformers and discovered that attention from the trivial sequence is dominating the final attention after accumulation, affecting its performance by including more noise than information on shallow layers. This issue is not handled by the attention function \eg Softmax. To solve this challenge, we propose to handle trivial weights explicitly by first separating out trivial attention weights with a relative threshold to the maximum attention and then adjusting them to a portion of the maximum attention weight. Experiments show up to $2.3\%$ increase in accuracy, indicating this process is necessary to make the attention function work.

\section*{Acknowledgement}

This work was partly supported in part by the Natural Sciences and Engineering Research Council of Canada (NSERC) under grant nos. RGPIN-2021-04244 and ALLRP 576612-22, and the United States Department of Agriculture (USDA) under grant no. 2019-67021-28996. This work was also supported in part by the National Natural Science Foundation of China (No. 62106267).

\balance
{\small
\bibliographystyle{ieee_fullname}
\bibliography{egbib}
}

\end{document}